\documentclass[11pt,a4paper]{article}

\usepackage[utf8]{inputenc}  
\usepackage[T1]{fontenc}     
\usepackage[english]{babel}  
\usepackage{amsmath, amssymb, amsthm}  
\usepackage{mathtools}
\usepackage{subcaption}
\usepackage{geometry}        
\usepackage{enumitem}
\usepackage{authblk}
\usepackage{bbold}
\usepackage{graphicx}        
\usepackage{hyperref}        
\usepackage{xcolor}          
\setlist[itemize]{topsep=-5pt, parsep=0pt, itemsep=2pt, leftmargin = 45pt} 

\usepackage{dsfont}

\setlist{nolistsep}
\usepackage{hyperref}
\newtheorem{theorem}{Theorem}
\newtheorem{lemma}{Lemma}
\newtheorem{corollary}{Corollary}
\newtheorem{proposition}{Proposition}
\newtheorem{remark}{Remark}
\newtheorem{assumption}{Assumption}

\newcommand{\R}{\mathbb{R}}  
\newcommand{\N}{\mathbb{N}}  
\newcommand{\E}{\mathbb{E}}  
\newcommand{\hth}{\hat{\theta}} 
\newcommand{\Tr}{\text{Tr}} 
\newcommand{\diff}{\textrm{d}} 

\newcommand{\DS}{\displaystyle}
\newcommand{\nt}{\lfloor nt \rfloor}
\newcommand{\neps}{\lfloor n\epsilon \rfloor}

\title{Functional Central Limit Theorem for Stochastic Gradient Descent}
\author[1]{Kessang Flamand}
\author[1]{Victor-Emmanuel Brunel}

\affil[1]{CREST-ENSAE}

\affil[ ]{\texttt{kessang.flamand@ensae.fr, victor-emmanuel.brunel@ensae.fr}}

\date{}

\begin{document}

\maketitle

\begin{abstract}
  We study the asymptotic shape of the trajectory of the stochastic gradient descent algorithm applied to a convex objective function. Under mild regularity assumptions, we prove a functional central limit theorem for the properly rescaled trajectory. Our result characterizes the long-term fluctuations of the algorithm around the minimizer by providing a diffusion limit for the trajectory. In contrast with classical central limit theorems for the last iterate or Polyak–Ruppert averages, this functional result captures the temporal structure of the fluctuations and applies to non-smooth settings such as robust location estimation, including the geometric median.
\end{abstract}


\section{Introduction}

\subsection{Framework}

In this work, we are interested in the asymptotic properties of the whole trajectory of stochastic algorithms for the minimization of convex objectives. Namely, let $\Phi:\R^d\to\R$ ($d\geq 1$) be a convex function. Let $G:\R^d\to\R^d$ be a measurable function such that $G(\theta)$ is a sub-gradient of $\Phi$ at $\theta$, for all $\theta\in\R^d$. If 
$\Phi$ is differentiable, then $G$ is simply given by $\nabla\Phi$. We consider algorithms based on such iterations as:
\begin{equation} \label{eqn:SGD_general}
    \begin{cases}
        \theta_0\in\R^d; \\
        \theta_{n}=\theta_{n-1}-t_n G_n, \quad \forall n\geq 1,
    \end{cases}
\end{equation}
where $\theta_0\in\R^d$ is the algorithm initialization, which we will always consider fixed, non-random for simplicity; $(t_n)_{n\geq 1}$ is a sequence of (deterministic) step-sizes; for all $n\geq 1$, $G_n$ is a noisy version of $G(\theta_{n-1})$ that can be written as $G_n=G(\theta_{n-1})+\varepsilon_n$ for some random vector $\varepsilon_n$  satisfying $\E[\varepsilon_n|\theta_{n-1}]=0$ almost surely. 

More precisely, we focus on the case when $\Phi$ can be expressed as the expectation of a convex loss:
\begin{equation}
    \Phi(\theta)=\E[\phi(X,\theta)], \quad \forall\theta\in\R^d,
\end{equation}
where $X$ is a random variable taking values in some abstract measurable space $(E,\mathcal E)$ and $\phi:E\times\R^d\to\R$ is a given map that is measurable in its first argument, convex in its second, and satisfies that $\phi(X,\theta)$ is integrable for all $\theta\in\R^d$. In that case, \cite[Theorem 2]{article:Brunel-2025} shows that there exists a map $g:E\times\R^d\to\R^d$ that is measurable in its first argument and satisfies that with probability $1$, $g(X,\theta)$ is a subgradient of $\phi(X,\cdot)$ at $\theta$, for all $\theta\in\R^d$. By setting $G(\theta)=\E[g(X,\theta)]$ for all $\theta\in\R^d$, \cite[Theorem 3]{article:Brunel-2025} ensures that the function $G$ is well defined, it is measurable and $G(\theta)$ is a subgradient of $\Phi$ at $\theta$, for all $\theta\in\Theta$. Thus, given i.i.d random variables $X_1,X_2,\ldots$ with the same distribution as $X$, we can set $G_n=g(X_n,\theta_{n-1})$ for all $n\geq 1$, so $\varepsilon_n:=G_n-G(\theta_{n-1})$ satisfies $\E[\varepsilon_n|\theta_{n-1}]=0$. 

In an offline context, the estimation of a minimizer $\theta^*$ of $\Phi$ based on i.i.d samples $X_1,\ldots,X_n$ typically resorts to $M$-estimation, or empirical risk minimization, where one seeks for a minimizer of the empirical loss $n^{-1}\sum_{i=1}^n\phi(X_i,\theta), \theta\in \R^d$ \cite{book:Hubert,haberman1989concavity,article:Niemiro-1992,article:Koltchinskii-1997,article:Brunel-2025}. Here, we consider the online problem, where the algorithm takes one data $X_n$ at a time to update its output $\theta_n$.

\subsection{Contributions}

Under minimal convexity assumtions on $\Phi$, we obtain a functional central limit theorem (FCLT) for the trajectories of the stochastic gradient descent (SGD) iterates \eqref{eqn:SGD_general}. Over classical central limit theorems for last iterates (or Polyak-Ruppert averages), our result allows to recover information on the fluctuations of the trajectory in long-term regimes. Notably, compared to classical asymptotic results on SGD, we do not require global strong convexity on $\Phi$. We show that local strong convexity on an arbitrarily small neighborhood of the minimizer suffices, encompassing situations such as geometric median estimation or, more generally, robust location estimation. 

\subsection{Related work}

Stochastic gradient descent (SGD), or Robbins-Monro procedure, has been widely studied since its first introduction in \cite{article:RobbinsMonro-1951}. The first central limit theorem (CLT) is from \cite{article:Chung-1954} and gives an  $n^{-1/2}$ rate of convergence when the step-size is of the form $cn^{-1}$ for some specific choice of $c>0$. For larger step-sizes $t_n = cn^{-\alpha}$ with $1/2<\alpha<1$, a CLT is obtained with rate $n^{-\alpha/2}$. Later, \cite{article:Sacks-1958} then \cite{article:Fabian-1968} obtained a CLT in the case $t_n = cn^{-1}$ using other methods that allow some assumptions to be relaxed
Later, \cite{article:Sacks-1958} and \cite{article:Fabian-1968} established a central limit theorem for step-sizes of the form $t_n = c/n$ that allowed for weaker assumptions on the moments of the noise and extended the results to the multidimensional setting. \cite{article:Fabian-1968} further highlighted the crucial role of the step-size constant $c$: if $c$ is too small, convergence can be arbitrarily slow, while for larger values ensuring a CLT at rate $n^{1/2}$, the asymptotic variance increases with $c$. The optimal value for $c$ depends on the Hessian of the objective function at the minimizer, so its calibration requires prior knowledge of the local curvature. Ever since, other CLT-type results have been shown, such as in the case of multiple targets in \cite{article:Pelletier-1998} or, very recently, infinite variance of the noise in the evaluation of gradients in \cite{article:Blanchet-2025}. In these two cases, the asymptotic distribution is not normal but is given as the stationary law of a stochastic process. Later, \cite{article:Polyak-1992} showed that the average of the first $n$ SGD iterates with large step-size converges at rate $n^{-1/2}$ with optimal asymptotic variance and step-size that does not require prior information. This version of SGD has also been widely studied, for example in very recent works in the specific case of geometric medians \cite{article:Cardotal-2025} -- where the objective is neither smooth nor strongly convex -- or even in non-convex cases \cite{article:Dereich-2023}. 

Results of functional type, that is, describing the distribution of the entire trajectory of SGD, are much less common in the literature. Notably, \cite{article:Berger-1997} established an almost sure invariance principle in the special case of linear filtering and regression, thereby providing the asymptotic temporal correlations of the iterates. In this specific linear setting, the problem can be reduced to the study of a system of linear equations whose coefficients are observed with noise. \cite{article:Arnaudonall_2012} established a functional central limit theorem in the context of barycenter estimation on a Riemannian manifold, using a Riemannian version of stochastic gradient descent. Their result is based on a convergence theorem for Markov chains to diffusion processes (Theorem 11.2.3 in \cite{book:StroockVaradhan}), which is based on the convergence of the Markov transition operators to the generator of the limiting diffusion -- a result that we also employ in our analysis. The authors assume bounded noise in the gradient evaluations and strong convexity of the objective function. Notably, Donsker's theorem \cite[Theorem 8.2]{book:Billingsley1968} provides a similar result for the case of mean estimation in Euclidean spaces.

Non-asymptotic properties of the SGD have also raised a lot of interest. First, \cite{book:NemirovskiYudin-1983,article:Nemirovskial-2009} give bounds on the $L^2$ distance between the $n$-th iteration of the SGD and the target minimizer for any finite $n$. Different rates of convergence have then been obtained in various settings. In particular, some bounds depending on the choice of the step-size were obtained in \cite{article:MoulinesBach-2011}, and later extended to the non-strongly convex case in \cite{article:BachMoulines-2013}. These bounds are typically written as the sum of a term that depends on the initial condition $\theta_0$ and another term that does not. This decomposition highlights the two regimes that characterize the convergence of stochastic gradient descent: a first regime in which the iterates move on average in the direction of the minimizer and approach it rapidly, followed by a second regime in which the iterates are close to the minimizer and fluctuate around it without any global preferred direction and with a decreasing variance. In particular, the trajectory in the first regime strongly depends on the initial condition, whereas this dependence is lost in the second regime. \cite{article:MoulinesBach-2011} showed that in the strongly convex case with $t_n = c/n$ for some $c>0$, the initial conditions are forgotten at rate $O(1/n^{\alpha/2})$ with $\alpha>1$ depending on $c$. In the case of bounded gradients, it also gives us that the fluctuations around the minimizer are of order $O(1/\sqrt{n})$, which is consistent with the central limit theorem of \cite{article:Chung-1954}, \cite{article:Sacks-1958}, \cite{article:Fabian-1968} and with our own result. In particular, our functional central limit theorem aims to quantify the fluctuations in the second regime.

For a version of SGD with constant step-sizes, iterates remain too noisy and do not converge to the minimizer $\theta^*$. Instead, they converge in distribution to the unique stationary distribution (\cite{article:Dieuleveut-al-2020}), which then collapses to $\theta^*$ as the constant step-size is chosen arbitrarily close to zero. It has been shown (\cite{article:Mandt-al-2015}, \cite{article:Li-al-2019}) that stochastic gradient algorithms with constant step-size can be approximated by a stochastic differential equation of the form
$\mathrm{d}X_t = -\nabla\Phi(X_t)\mathrm{d}t + \sqrt{\eta}\Sigma(X_t)\mathrm{d}B_t$, where $\eta$ is the constant step-size and $\Sigma(X_t)$ is a covariance matrix. This result quantifies the fluctuations of the SGD trajectory around that of the deterministic gradient descent, showing that the diffusive term vanishes as the step-size (and hence the noise magnitude) goes to zero. Consequently, this result is fundamentally different from our FCLT, which describes the fluctuations of SGD with a decreasing step-size around the minimizer $\theta^*$, in a regime where the iterates are already close to the optimum and the gradient of $\Phi$ is well approximated by its first-order Taylor expansion.

\section{Main results}

Before stating our results, we first give our main working assumptions.

\begin{assumption} \label{assump:Phi-a}
    The function $\Phi$ has a unique minimizer $\theta^*$. 
\end{assumption}

\begin{assumption} \label{assump:Phi-b}
    The function $\Phi$ has a unique minimizer $\theta^*$ and it is twice continuously differentiable in a neighborhood of $\theta^*$ with positive definite Hessian $\nabla^2\Phi(\theta^*)$ at $\theta^*$.
\end{assumption}

\begin{assumption}\label{assump:quadratic_growth}
    There exists $L>0$ such that $\|G(\theta)\|\leq L\|\theta-\theta^*\|$, for all $\theta\in\R^d$.
\end{assumption}

This assumption states that $\Phi$ grows at most quadratically from $\theta^*$. Note that we do not require $\Phi$ to be differentiable and to have Lipschitz gradients. In the absence of further regularity assumption on $\Phi$, Assumption~\ref{assump:quadratic_growth} is necessary even to obtain convergence of SGD, since convergence could fail otherwise even in the noiseless case. 

\begin{assumption}\label{assump:noise1}
    There exists $\sigma^2>0$ such that $\E[\|g(X_1,\theta)-G(\theta)\|^2]\leq\sigma^2$ for all $\theta\in\R^d$. 
\end{assumption}

This assumption states that for any query of a subgradient of $\Phi$, the error always has a second moment that is uniformly bounded, irrespective of the point $\theta\in\R^d$. In particular, under this assumption, independence of the $X_i$'s yields that $\E[\|g(X_n,\theta_{n-1})-G(\theta_{n-1})\|^2|\mathcal F_{n-1}]\leq\sigma^2$ almost surely, for all $n\geq 1$, which is a common assumption. 

The next assumption is more stringent as it requires that, at least in a neighborhood of $\theta^*$, that error is uniformly bounded in an $L^2$ sense.

\begin{assumption}\label{assump:noise2}
    There exists $\eta>0$ such that $\E[\sup_{\theta\in B(\theta^*,\eta)}\|g(X_1,\theta)-G(\theta)\|^2]<\infty$.
\end{assumption}

Below, we explain how Assumption~\ref{assump:noise2} can be replaced by a set of two assumptions which, in some case, may be less restrictive (see Assumptions~\ref{assump:noise21} and \ref{assump:noise22}).

Our first theorem is the consistency of the sequence $(\theta_n)_{n\geq 0}$ defined in \eqref{eqn:SGD_general}, provided the step-sizes are chosen appropriately. For all $n\geq 1$, we denote by $\mathcal F_n$ the $\sigma$-algebra spanned by $X_1,\ldots,X_n$ and by $\mathcal F_0$ the trivial $\sigma$-algebra (we are implicitly assuming that all $X_n$'s are defined on some probability space $(\Omega,\mathcal F,P)$ -- the $\sigma$-algebras $\mathcal F_0,\mathcal F_1,\ldots$ are included in $\mathcal F$). 

\begin{theorem} \label{thm:consistency}
    Let the sequence of step-sizes $(t_n)_{n\geq 1}$ satisfy $\sum_{n\geq 1}t_n=\infty$ and $\sum_{n\geq 1}t_n^2<\infty$. Let Assumptions~\ref{assump:Phi-a}, \ref{assump:quadratic_growth} and \ref{assump:noise1} hold. Then, $\theta_n\xrightarrow[n\to\infty]{} \theta^*$ almost surely.
\end{theorem}

\begin{proof}
    For all $n\geq 1$, denote by $G_n=g(X_n,\theta_{n-1})$ and by $\varepsilon_n=G_n-G(\theta_n)$. Assumption~\ref{assump:noise1} yields that 
    \begin{equation} \label{eqn:bound_noise}
        \E[\|\varepsilon_n\|^2|\mathcal F_{n-1}]\leq \sigma^2
    \end{equation}
    for all $n\geq 1$. Therefore, 
    \begin{align*}
        \|\theta_n-\theta^*\|^2 & = \|\theta_{n-1}-t_nG_n-\theta^*\|^2 \\
        & = \|\theta_{n-1}-\theta^*\|^2-2t_n(\theta_{n-1}-\theta^*)^\top G_n+t_n^2\|G_n\|^2 \\
        & = \|\theta_{n-1}-\theta^*\|^2-2t_n(\theta_{n-1}-\theta^*)^\top G(\theta_{n-1})-2t_n(\theta_{n-1}-\theta^*)^\top\varepsilon_n \\
        & \quad \quad +t_n^2\|G(\theta_{n-1})\|^2+2t_n^2G(\theta_{n-1})^\top\varepsilon_n+t_n^2\|\varepsilon_n\|^2 \\
        & \leq (1+Lt_n^2)\|\theta_{n-1}-\theta^*\|^2-2t_n(\Phi(\theta_{n-1})-\Phi(\theta^*))-2t_n(\theta_{n-1}-\theta^*)^\top\varepsilon_n \\
        & \quad \quad +2t_n^2G(\theta_{n-1})^\top\varepsilon_n+t_n^2\|\varepsilon_n\|^2
    \end{align*}
    where we used the convexity of $\Phi$ and Assumption~\ref{assump:quadratic_growth} in the last inequality. Taking the conditional expectation given $\mathcal F_{n-1}$ on both sides yields
    \begin{equation} \label{eq:prop1_1}
        \E[\|\theta_n-\theta^*\|^2|\mathcal F_{n-1}] \leq (1+Lt_n^2)\|\theta_{n-1}-\theta^*\|^2-2t_n(\Phi(\theta_{n-1})-\Phi(\theta^*))+t_n^2\sigma^2
    \end{equation}
    thanks to \eqref{eqn:bound_noise}. Now, by Robbins-Siegmund theorem \cite{article:Robbins-1971}, the sequence $(\|\theta_n-\theta^*\|)_{n\geq 0}$ must converge almost surely to a non-negative random variable $Z$, and the sum $\sum_{n\geq 1}t_n(\Phi(\theta_{n-1})-\Phi(\theta^*))$ must be finite with probability $1$. By Lemma~\ref{lemma:convex-RS}, it must hold that $Z=0$ almost surely, hence, $\theta_n\xrightarrow[n\to\infty]{}\theta^*$ almost surely.
\end{proof}

Now, if we also assume Assumption~\ref{assump:Phi-b}, we have the following result, which is essential for our main theorem. 

\begin{theorem} \label{thm:tightness}
    Let Assumptions~\ref{assump:Phi-b}, \ref{assump:quadratic_growth} and \ref{assump:noise1} hold. Then, the sequence $(\sqrt n(\theta_n-\theta^*))_{n\geq 0}$ is tight.
\end{theorem}

We defer the proof of this theorem to the appendix. The next proposition introduces a diffusion process, which will be the limit of a rescaled version of the trajectory of the sequence $(\theta_n)_{n\geq 0}$.

\begin{proposition} \label{prop: sol EDS}
    Let $H,\Sigma\in\R^{d\times d}$ be positive semi-definite, symmetric matrices. Assume that $H$ is invertible and let $0<\mu_1\leq\ldots\leq\mu_d$ be its eigenvalues and $(e_1,\ldots,e_n)$ an associated orthonormal basis of eigenvectors. For all $t>0$ and continuously differentiable functions $f:\R^d\to\R$, let $G_tf:\R^d\to\R$ be the function defined as
    $$G_tf(y)=-t^{-1}y^\top H\nabla f(y)+\frac{1}{2}\Tr(\Sigma \nabla^2 f(y)), \quad \forall y\in\R^d.$$
    Let $(B_t)_{t\geq 0}$ be a Brownian motion adapted to a filtration $\mathcal A=(\mathcal A_t)_{t\geq 1}$. 
    There exists a unique diffusion process $Y$ $\mathcal A$-adapted (up to indistinguishability) whose generator is given by $(G_t)_{t>0}$ and that satisfies $Y_t\xrightarrow[t\downarrow 0]{}0$ in probability. Moreover, $Y_t\xrightarrow[t\downarrow 0]{}0$ almost surely and one can write, with probability $1$, that for all $t>0$, 
    $$Y_t=\int_0^t e^{\log(s/t)H}\Sigma^{1/2}\diff B_s = \sum_{i=1}^d t^{-\mu_i}\left(e_i^\top \int_0^t s^{\mu_i}\Sigma^{1/2}\diff B_s\right)e_i.$$
\end{proposition}

By \cite[Theorem 7.3.3]{book:Oksendal-2003}, the process $(Y_t)_{t>0}$ must satisfy the stochastic differential equation
$$
\diff Y_t = -t^{-1}HY_t\diff t + \Sigma^{1/2}\mathrm{d}B_t, \quad \forall t>0.
$$

In the sequel, we let Assumption~\ref{assump:Phi-b} hold. We denote by $0<\lambda_1\leq\ldots\leq \lambda_d$ the eigenvalues of $\nabla^2\Phi(\theta^*)$, with associated orthonormal basis of eigenvectors $(e_1,\ldots,e_n)$. We also assume that $g(X_1,\theta^*)$ has two moments and we let $\Gamma=\E[g(X_1,\theta^*)g(X_1,\theta^*)^\top]$ be its covariance matrix. Note that $\E[g(X_1,\theta^*)]=G(\theta^*)=\nabla\Phi(\theta^*)=0$, since Assumption~\ref{assump:Phi-b} implies differentiability of $\Phi$ at $\theta^*$. Moreover, Assumption~\ref{assump:Phi-b} implies differentiability of $\Phi$ at $\theta^*$, so $\phi(X_1,\cdot)$ must be almost surely differentiable at $\theta^*$, by \cite[Lemma 13]{article:Brunel-2025}. Hence, the choice of $g(X_1,\theta^*)$ in the definition of $\Gamma$ is almost surely unique. 

Fix a number $\delta>1/\lambda_1$ and consider step-sizes given by $t_n=\delta/n$, for all $n\geq 1$. 

Applying Proposition~\ref{prop: sol EDS} to $H=\delta\nabla^2\Phi(\theta^*)-I_d$ and $\Sigma=\delta^2\Gamma$ yields the existence of a unique (up to indistinguishability) diffusion process $Y$ in $\mathcal C^0((0,\infty),\R^d)$ with generator $(G_t)_{t>0}$ given by 
\begin{equation} \label{eq:generator}
    G_tf(y) = t^{-1}y^\top (I_d-\delta\nabla^2\Phi(\theta^*))\nabla f(y) + \frac{\delta^2}{2}\Tr(\Gamma\nabla^2f(y)), \quad y\in\R^d, t>0, f\in \mathcal C^2(\R^d,\R)
\end{equation}
and satisfying both
\begin{equation} \label{eqn:diffusion}
\diff Y_t=t^{-1}(I_d-\delta\nabla^2\Phi(\theta^*))Y_t\diff t+\delta\Gamma^{1/2}\diff B_t,\quad \forall t>0,
\end{equation}
where $(B_t)_{t\geq 0}$ is a standard Brownian motion, and $Y_t\xrightarrow[t\downarrow 0]{} 0$ in probability.
In the following, we let $a(t,y)=t^{-1}(I_d-\delta\nabla^2\Phi(\theta^*))y$ and $b(t,y)=\delta^2\Gamma$, for all $(t,y)\in (0,\infty)\times\R^d$, which we refer to as the \textit{drift} and \textit{diffusion} terms of $(Y_t)_{t>0}$, respectively.

Now, we define rescaled, continuous time trajectories of the sequence of iterates $(\theta_n)_{n\geq 1}$ as follows. 

For every $n\geq 1$, consider the sequence $(\tilde Y_{k}^n)_{k\geq 0}$ defined by setting 
\begin{equation}
\label{eq:def_rescaling}
    \tilde Y_k^n=\frac{k}{\sqrt n}(\theta_k-\theta^*), \quad \forall k\geq 0.
\end{equation}
Then, $(\tilde Y_k^n)_{k\geq 0}$ is an inhomogeneous, discrete time Markov chain. We introduce the continuous time process $(Y_t^n)_{t>0}$ by linear interpolation of the values of $\tilde Y_{k}^n$, that is, we set, for all $t>0$, 
$$Y_t^n=(k-nt)\tilde Y_{k-1}^n+(nt-k+1)\tilde Y_{k}^n, \quad \quad (k-1)/n<t\leq k/n.$$
By construction, $Y^n$ is a continuous map defined on $(0,\infty)$ and taking values in $\R^d$. We denote by $\mathcal C^0((0,\infty),\R^d)$ the space of all such maps, which we equip with the topology induced by uniform convergence on compact intervals of $(0,\infty)$. 

Let us state our main result, providing a functional central limit theorem for $(Y^n)_{n\geq 1}$.

\begin{theorem} \label{thm:main}
    Let Assumptions~\ref{assump:Phi-b}, \ref{assump:quadratic_growth}, \ref{assump:noise1} and \ref{assump:noise2} hold and let $Y$ be the diffusion defined in \eqref{eqn:diffusion}. Then, 
    $$Y^n\xrightarrow[n\to\infty]{} Y$$
    in distribution in the space $\mathcal C^0((0,\infty),\R^d)$ endowed with the topology induced by uniform convergence on compact intervals. 
\end{theorem}

\begin{remark}The limiting law given by Theorem~\ref{thm:main} is centered and does not depend on the initial condition $\theta_0$. However, the first iterations of stochastic gradient descent generally depend on the initial condition and, in particular, there is no reason for them to be centered around the minimizer.  This is because the dependence on the initial condition decreases faster than the amplitude of the fluctuations (see, for instance, \cite{article:MoulinesBach-2011}). Thus, under our rescaling, the term depending on the initial condition disappears asymptotically, and the remaining fluctuations are centered and gaussian. The topology fixed on $\mathcal{C}^0((0,\infty),\R^d)$ must, however, account for the time required to forget the initial condition, and consequently cannot describe convergence on time intervals that include $0$.
\end{remark}

Figure \ref{fig:zoom-fluctuations} below illustrates the previous remark through a realization of stochastic gradient descent. The first panel shows the entire trajectory, exhibiting a first regime corresponding to a noisy gradient descent, followed by a second regime in which the iterates remain close to the minimizer. The second panel zooms in on this second regime and shows centered fluctuations of decreasing amplitude. The third panel shows the rescaled trajectory obtained via \eqref{eq:def_rescaling}. The resulting trajectory illustrates the result of Theorem \ref{thm:main}, as it exhibits a diffusion-like behavior with fluctuations of constant variance.
\bigskip

\begin{figure}[htbp]
  \centering
  \includegraphics[width=1\textwidth]{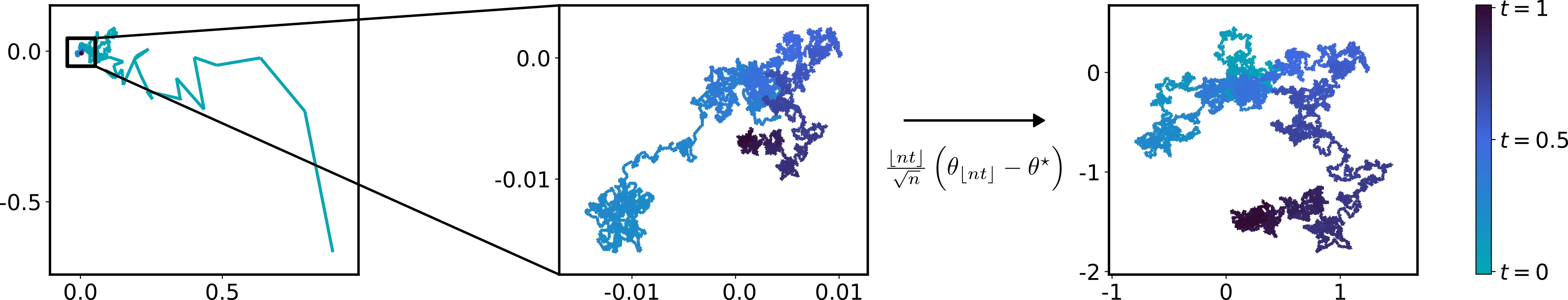}
  \caption{Stochastic gradient descent trajectory for the estimation of the median of a Laplace$(0,1)$ distribution in $\R^2$
  with independent coordinates, based on $n=50000$ samples. The first panel shows the full trajectory. The second panel zooms in on the fluctuations around the minimizer, with the first $2000$ iterations removed. The third panel shows the rescaled trajectory, illustrating the diffusion-like behavior predicted by Theorem \ref{thm:main}. Color indicates time, from light (start) to dark (end). We set the step size to $2/k$.} 
  \label{fig:zoom-fluctuations}
\end{figure}
In order to present a proof of this Theorem \ref{thm:main}, let us introduce some important quantities. 
For all $n\geq 1$, we let $(P_k^n)_{k\geq 1}$ be the sequence of transition kernels of the Markov chain $(\tilde Y_k^n)_{k\geq 0}$, that is, 
$$P_k^n(y,A)=P(\tilde Y_k^n\in A|\tilde Y_{k-1}^n=y)$$
for all $y\in\R^d$ and for all Borel sets $A$ in $\R^d$. Now, for all $t>0$ and $y\in\R^d$, we also define $a_n(t,y)$ and $b_n(t,y)$ as the (rescaled) first and second conditional moments of the increments:
$$a_n(t,y)=n\E\left[\tilde Y_{\lfloor nt\rfloor}^n-\tilde Y_{\lfloor nt\rfloor-1}^n|\tilde Y_{\lfloor nt\rfloor-1}^n=y\right]$$
and 
$$b_n(t,y)=n\E\left[(\tilde Y_{\lfloor nt\rfloor}^n-\tilde Y_{\lfloor nt\rfloor-1}^n)(\tilde Y_{\lfloor nt\rfloor}^n-\tilde Y_{\lfloor nt\rfloor-1}^n)^\top|\tilde Y_{\lfloor nt\rfloor-1}^n=y\right].$$

The next result, which is key to our main theorem, shows in particular that $a_n$ and $b_n$ converge uniformly on compact subsets of $(0,\infty)\times\R^d$ to the drift and diffusion terms $a$ and $b$ of $Y$ respectively.

\begin{proposition}
       \label{prop:cv coeff}
       Let Assumptions~\ref{assump:Phi-b}, \ref{assump:quadratic_growth} and \ref{assump:noise2} hold. Fix $R, r,\varepsilon,T$ with $0<\varepsilon<T$. Then, the following statements hold.
       \begin{itemize}
       \setlength{\itemsep}{0.6em}
           \item[(i)] $\DS \sup_{\substack{\varepsilon\leq t\leq T, \\ y\in B(\theta^*,
           R)}} nP^n_{\nt}\left(y,B(y, r)^\complement\right) \xrightarrow[n\to\infty]{} 0$;
           
           \item[(ii)] $\DS  \sup_{\substack{\varepsilon\leq t\leq T, \\ y\in B(\theta^*,
           R)}} \|a_n(t,y)-a(t,y)\| \xrightarrow[n\to\infty]{} 0$;
    
           \item[(iii)] $\DS \sup_{\substack{\varepsilon\leq t\leq T, \\ y\in B(\theta^*,
           R)}} \|b_n(t,y)-\delta^2\Gamma\| \xrightarrow[n\to\infty]{} 0$.
    \end{itemize}
\end{proposition}

\begin{proof}

First, for $n\ge1$ and $k\geq 1$, rewrite
\begin{align}
    \tilde Y_{k}^n & = \frac{k}{\sqrt{n}}(\hat{\theta}_{k}-\theta^*) \nonumber \\
     & = \frac{k}{\sqrt{n}}\left(\hat{\theta}_{k-1}-\frac{\delta}{k}g(X_{k},\hat{\theta}_{k-1}) - \theta^*\right) \nonumber \\
    & = \frac{k-1}{\sqrt{n}}(\hth_{k-1}-\theta^*)+ \frac{1}{\sqrt{n}}(\hth_{k-1}-\theta^*)-\frac{\delta}{\sqrt{n}}g(X_{k},\hth_{k-1}) \nonumber \\
    & = \tilde Y_{k-1}^n+\frac{\tilde Y_{k-1}^n}{k-1}-\frac{\delta}{\sqrt{n}}g(X_{k},\theta^*+\sqrt n \tilde Y_{k-1}^n/(k-1)).\label{eq:dvlpY} 
\end{align}

Now, fix $t\in [\varepsilon,T]$ and $y\in B(\theta^*,R)$. First, note that using \eqref{eq:dvlpY} and independence of the $X_i$'s, $P^n_{\nt}\left(y,B(y,r)^\complement\right)$ can be rewritten as 
\begin{equation} \label{eqn:transition1}
    P^n_{\nt}\left(y,B(y,r)^\complement\right) = P\left(\left\|\frac{y}{\nt}-\frac{\delta}{\sqrt{n}}g\left(X_{\nt},\theta^*+\frac{\sqrt{n}}{\nt-1}y\right)\right\|>r\right).
\end{equation}
Now, by the triangle inequality, 
\begin{align*}
    & \left\|\frac{y}{\nt}-\frac{\delta}{\sqrt{n}}g\left(X_{\nt},\theta^*+\frac{\sqrt{n}}{\nt-1}y\right)\right\| \\
    & \quad \quad \quad \quad \leq \frac{\|y\|}{\nt}+\frac{\delta}{\sqrt{n}}\left\|g\left(X_{\nt},\theta^*+\frac{\sqrt{n}}{\nt-1}y\right)\right\| \\
    & \quad \quad \quad \quad \leq \frac{\|y\|}{\nt}+ \frac{\delta}{\sqrt n}\left\|G\left(\theta^*+\frac{\sqrt n}{\nt-1}y\right)\right\| \\ 
    & \quad \quad \quad \quad \quad \quad \quad \quad \quad +\frac{\delta}{\sqrt n}\left\|g\left(X_{\nt},\theta^*+\frac{\sqrt{n}}{\nt-1}y\right)-G\left(\theta^*+\frac{\sqrt n}{\nt-1}y\right)\right\| \\ 
    & \quad \quad \quad \quad \leq \frac{(1+\delta L)\|y\|}{\nt-1}+\frac{\delta}{\sqrt n}\left\|g\left(X_{\nt},\theta^*+\frac{\sqrt{n}}{\nt-1}y\right)-G\left(\theta^*+\frac{\sqrt n}{\nt-1}y\right)\right\| \\
    & \quad \quad \quad \quad \leq \frac{(1+\delta L)(\|\theta^*\|+R)}{\neps-1}+\frac{\delta}{\sqrt n}\left\|g\left(X_{\nt},\theta^*+\frac{\sqrt{n}}{\nt-1}y\right)-G\left(\theta^*+\frac{\sqrt n}{\nt-1}y\right)\right\|
\end{align*}
where we used Assumption~\ref{assump:quadratic_growth} in the inequality before the last. Therefore, if $n$ is sufficiently large so the first term in the last line is not larger than $r/2$, this and \eqref{eqn:transition1} yield 
\begin{align}
     P^n_{\nt}\left(y,B(y,r)^\complement\right) & \leq P\left(\left\|g\left(X_{\nt},\theta^*+\frac{\sqrt{n}}{\nt-1}y\right)-G\left(\theta^*+\frac{\sqrt{n}}{\nt-1}y\right)\right\|>\frac{r\sqrt n}{2\delta}\right) \nonumber \\
     & = P\left(\left\|g\left(X_{1},\theta^*+\frac{\sqrt{n}}{\nt-1}y\right)-G\left(\theta^*+\frac{\sqrt{n}}{\nt-1}y\right)\right\|>\frac{r\sqrt n}{2\delta}\right). \label{eqn:bound_V}
\end{align}
Now, further assume that $n$ is sufficiently large so $\sqrt n R/(n\varepsilon-1)\leq\eta$, where $\eta$ is defined in Assumption~\ref{assump:noise2}. Letting $V=\sup_{\theta\in B(\theta^*,\eta)}\|g(X_1,\theta)-G(\theta)\|^2$, we obtain that $nP^n_{\nt}\left(y,B(y,r)^\complement\right)\leq nP(V>r^2n/(2\delta)^2)$, which goes to $0$ as $n\to\infty$ by Lemma~\ref{lemma:moments}. This proves the first statement. 

For the second statement, fix again $t\in [\varepsilon,T]$ and $y\in B(\theta^*,R)$ and let $n\geq 1$ be sufficiently large so $n\varepsilon\geq 2$, use independence of the $X_i$'s and \eqref{eq:dvlpY} to write $a_n(t,y)$ as
\begin{align*}
    a_n(t,y) & = n\E\left[\frac{y}{\nt-1}-\frac{\delta}{\sqrt n}g(X_{\nt},\theta^*+\sqrt ny/(\nt-1))\right] \\
    & = \frac{ny}{\nt-1}-\sqrt n\delta G(\theta^*+\sqrt ny/(\nt-1)).
\end{align*}
Now, Assumption~\ref{assump:Phi-b} implies that $\Phi$ is differentiable in a neighborhood of $\theta^*$, yielding that $G(\theta^*+\sqrt ny/(\nt-1))=\nabla\Phi(\theta^*+\sqrt ny/(\nt-1))$. Moreover, $\Phi$ is twice continuously differentiable in a neighborhood of $\theta^*$ and $\nabla\Phi(\theta^*)=0$, so $\sqrt n\nabla\Phi(\theta^*+y/(\nt-1))\xrightarrow[n\to\infty]{} t^{-1}\nabla^2\Phi(\theta^*)y$ uniformly in $t\in [\varepsilon,T]$ and $y\in B(\theta^*,R)$. Hence, 
$$a_n(t,y)\xrightarrow[n\to\infty]{} t^{-1}(y-\delta\nabla^2\Phi(\theta^*)y)=a(t,y)$$
uniformly in $t\in [\varepsilon,T]$ and $y\in B(\theta^*,R)$, which proves the second statement. 

Finally, for the third statement, again using \eqref{eq:dvlpY}, we write 
\begin{align*}
    nb_n(t,y) & = n\E\Bigg[\left(\frac{y}{\nt-1}-\frac{\delta}{\sqrt n}g\left(X_{1},\theta^*+\frac{\sqrt n}{\nt-1}y\right)\right) \\ 
    & \quad \quad \quad \quad \quad \quad \quad\left(\frac{y}{\nt-1}-\frac{\delta}{\sqrt n}g\left(X_{1},\theta^*+\frac{\sqrt n}{\nt-1}y\right)\right)^\top\Bigg] \\
    & = \frac{n}{(\nt-1)^2}yy^\top+\delta^2\E\left[g\left(X_{1},\theta^*+\frac{\sqrt n}{\nt-1}y\right)g\left(X_{1},\theta^*+\frac{\sqrt n}{\nt-1}y\right)^\top\right] \\ 
    & \quad \quad \quad -\frac{\sqrt n\delta}{\nt-1}yG\left(\theta^*+\frac{\sqrt n}{\nt-1}y\right)^\top -\frac{\sqrt n\delta}{\nt-1}G\left(\theta^*+\frac{\sqrt n}{\nt-1}y\right)y^\top.
\end{align*}
Using a similar argument as above, the last two terms go to $0$ uniformly in $t\in [\varepsilon,T]$ and $y\in B(\theta^*,R)$. It is clear that the first term does too. Hence, it is now sufficient to check that 
$$F(t,y):=\E\left[g\left(X_{1},\theta^*+\frac{\sqrt n}{\nt-1}y\right)g\left(X_{1},\theta^*+\frac{\sqrt n}{\nt-1}y\right)^\top\right]\xrightarrow[n\to\infty]{} \Gamma$$
uniformly in $t\in [\varepsilon,T]$ and $y\in B(\theta^*,R)$. First, by Lemma~\ref{lemma:unif_conv}, it is sufficient to show that for all sequences $(t_n)_{n\geq 1}\subseteq [\varepsilon,T]$ and $(y_n)_{n\geq 1}\subseteq B(\theta^*,R)$, $F(t_n,y_n)\xrightarrow[n\to\infty]{}\Gamma$. Note that for such sequences, we have that $\frac{\sqrt n}{\lfloor nt_n\rfloor-1}y_n\xrightarrow[n\to\infty]{}0$. Hence, let us simply consider any sequence $(u_n)_{n\geq 1}$ in $\R^d$ such that $u_n\xrightarrow[n\to\infty]{} 0$ and show that 
\begin{equation} \label{eqn:conv_2nd}
\E\left[g\left(X_{1},\theta^*+u_n\right)g\left(X_{1},\theta^*+u_n\right)^\top\right]\xrightarrow[n\to\infty]{} \Gamma.
\end{equation}
First, without loss of generality, assume that $\|u_n\|\leq\eta$ for all $n\geq 1$. Since $\phi(X_1,\cdot)$ is almost surely differentiable at $\theta^*$, \cite[Lemma 9]{article:Brunel-2025} yields that $g(X_1,\theta^*+u_n)\xrightarrow[n\to\infty]{} g(X_1,\theta^*)$ almost surely. Moreover, the operator norm of $g(X_1,\theta^*+u_n)g(X_1,\theta^*+u_n)^\top$ is given by $\|g(X_1,\theta^*+u_n)\|^2$, which is dominated by $\sup_{\theta\in B(\theta^*,\eta)}\|g(X_1,\theta)\|^2$ which, by Assumption~\ref{assump:noise2}, is integrable. Therefore, the dominated convergence theorem yields \eqref{eqn:conv_2nd}.

\end{proof}

We are now ready to prove Theorem~\ref{thm:main}. \newline

\begin{proof}[Proof of Theorem~\ref{thm:main}]

    First, for all positive integers $p$ and $n$, let $Y^{n,p}$ be the restriction of the stochastic process $Y^n$ to the interval $[1/p,\infty)$. Then, we have the following lemma, whose proof is deferred to the appendix.
    \begin{lemma} \label{lem:cv process}
        For all $p\geq 1$, the process $(Y^{n,p})_{n\geq 1}$ is tight in $\mathcal C^0([1/p,\infty),\R^d)$ equipped with the topology induced by uniform convergence on compact intervals. 
    \end{lemma}

We will now show that any subsequence of $(Y^n)_{n\ge 1}$ has a further subsequence that converges weakly to a process $(Z_t)_{t>0}$ with generator $G$ and that satisfies $Z_t\xrightarrow[t\downarrow 0]{} 0$ in probability. In order to avoid renumbering, let us simply show that $(Y^n)_{n\geq 1}$ has such a subsequence. 

Lemma~\ref{lem:cv process} shows the existence of a subsequence $(\tilde Y^{\phi_1(n),1})_{n\ge 1}$ of $(Y^{n,1})_{n\ge 1}$ that converges weakly in $\mathcal C^0([1,\infty),\R^d)$ to some process with generator $(G_t)_{t\geq 1}$. Similarly, one can extract a subsequence $(\tilde Y^{\phi_2(n),2})_{n\ge 1}$ of $(Y^{\phi_1(n),2})_{n\ge 1}$ that converges weakly in $\mathcal C^0([1/2,\infty),\R^d)$ to some process with generator $(G_t)_{t\geq 1/2}$. Reiterating this construction, for all integers $p\geq 1$, we can extract a subsequence $(\tilde Y^{\phi_p(n),p})_{n\ge 1}$ of $(Y^{\phi_{p-1}(n),p})_{n\ge 1}$ that converges weakly in $\mathcal C^0([1/p,\infty),\R^d)$ to some process with generator $(G_t)_{t\geq 1/p}$, for all integers $p\geq 1$. Now, fix $\varepsilon>0$ and consider the subsequence $(Y^{\phi_n(n)})_{n\geq 1}$ of $(Y^n)_{n\geq 1}$. Let $p\geq 1$ be a sufficiently large integer such that $1/p\leq\varepsilon$. Then -- except maybe for the first terms -- the restriction of $(Y^{\phi_n(n)})_{n\geq 1}$ to $[1/p,\infty)$ is a subsequence of $(Y^{\phi_p(n),p})_{n\geq 1}$, hence, it converges weakly in $\mathcal C^0([1/p,\infty),\R^d)$ to some process with generator $(G_t)_{t\geq 1/p}$. Therefore, the restriction of $(Y^{\phi_n(n)})_{n\geq 1}$ to $[\varepsilon,\infty)$ converges weakly in $\mathcal C^0([\varepsilon,\infty),\mathbb{R}^d)$ to some process with generator $(G_t)_{t\geq \varepsilon}$. In particular, the restriction of $(Y^{\phi_n(n)})_{n\geq 1}$ to $[\varepsilon,\infty)$ is tight in $\mathcal{C}^0([\epsilon,\infty),\mathbb{R}^d)$ for every $\varepsilon>0$. Hence Lemma~\ref{lemma:weak-conv-eps} ensures that $(Y^{\phi_n(n)})_{n\geq 1}$ is tight in $\mathcal{C}^0((0,\infty),\mathbb{R}^d)$. 
Moreover, if $Z$ is an accumulation point of $(Y^{\phi_n(n)})_{n\geq 1}$ then $Z$ is a process with generator $G$. To conclude, we need to show that $Z_t\xrightarrow[t\downarrow 0]{} 0$ in probability. Consider a subsequence $(Y^{\psi(n)})_{n\ge1}$ of $(Y^{\phi(n)})_{n\ge1}$ that converges in distribution to $Z$. First, note that for all $t>0$, $Y_t^{\psi_n(n)}\xrightarrow[n\to\infty]{} Z_t$ in distribution, since the map $g\in \mathcal C^0((0,\infty),\R^d)\mapsto g(t)$ is continuous for the topology induced by uniform convergence on compact intervals. Therefore, for all $\alpha>0$ and $t>0$,
\begin{align*}
    P(\|Z_t\|>\alpha) & = \lim_{n\to\infty} P(\|Y_t^{\psi_n(n)}\|>\alpha) \\
    & \leq \sup_{n\geq 1} P(\sqrt t \sqrt{n}\|\theta_n-\theta^*\|>\alpha) \xrightarrow[t\downarrow 0]{} 0
\end{align*}
by Theorem~\ref{thm:tightness}. Therefore, $Z_t\xrightarrow[t\downarrow 0]{}0$ in probability. Let $Y$ be the process defined in \eqref{eqn:diffusion}. Then, the uniqueness statement in Proposition~\ref{prop: sol EDS} ensures that $Z$ and $Y$ are indistinguishable. Hence, $Y$ is the unique accumulation point of $(Y^{\phi(n)})_{n\ge1}$ which is tight, and the result follows.
\end{proof}

For the proof of Theorem~\ref{thm:main}, we could in fact replace Assumption~\ref{assump:noise2}, which may be quite stringent in certain situations, with the following two assumptions.

\begin{assumption} \label{assump:noise21}
    There exist $\eta,\varepsilon,M>0$ such that $\E[\|g(X_1,\theta)-G(\theta)\|^{2+\varepsilon}]\leq M$ for all $\theta\in B(\theta^*,\eta)$.
\end{assumption}

\begin{assumption} \label{assump:noise22}
    The map $\theta\in\R^d\mapsto\E[g(X_1,\theta)g(X_1,\theta)^\top]$ is continuous at $\theta^*$.
\end{assumption}

Indeed, Assumption~\ref{assump:noise2} was used in two places. First, to prove that the right hand side in \eqref{eqn:bound_V} goes to $0$ as $n\to\infty$ (via Lemma~\ref{lemma:moments}), which would still hold under Assumption~\ref{assump:noise1}. Second, to prove \eqref{eqn:conv_2nd}, which would be a direct consequence of Assumption~\ref{assump:noise2}. 
Moreover, note that under Assumptions~\ref{assump:Phi-b} and \ref{assump:quadratic_growth}, Assumption~\ref{assump:noise2} is equivalent for the map $\theta\in\R^d\mapsto \textrm{Var}(g(X_1,\theta))$ to be continuous at $\theta^*$, where $\textrm{Var}(g(X_1,\theta))$ is simply the covariance matrix of the noise term in the SGD step when the subgradient $G$ of $\Phi$ is evaluated at $\theta$.

As a consequence of Theorem~\ref{thm:main}, we obtain the asymptotic normality of the iterates $(\theta_n)_{n\geq 1}$ of the stochastic algorithm. We defer its proof to the appendix.

\begin{corollary}
\label{cor:CLT}
    Under the same assumptions as in Theorem~\ref{thm:main}, we have that
    $$
    \sqrt{n}(\hth_n-\theta^*)\xrightarrow[n\to\infty]{}\mathcal{N}_d(0,\Sigma)
    $$
    where $\Sigma = \delta\displaystyle\int_0^\infty  e^{t/\delta}e^{-t\nabla^2\Phi(\theta^*)} \Gamma e^{-t\nabla^2\Phi(\theta^*)} \diff t$.
\end{corollary}

For $n\ge1$, define $\hat\theta_n$ as a (measurable) minimizer of the (offline) empirical risk $\theta\in\R^d\mapsto \frac{1}{n}\sum_{i=1}^n\phi(X_i,\theta)$. It is known that under Assumptions~\ref{assump:Phi-b} and \ref{assump:noise1}, 
\begin{equation*}
    \sqrt n(\hat\theta_n-\theta^*)\xrightarrow[n\to\infty]{} \mathcal N_d(0,\Delta)
\end{equation*}
in distribution, where $\Delta=\nabla^2\Phi(\theta^*)^{-1}\Gamma \nabla^2\Phi(\theta^*)^{-1}$ \cite{haberman1989concavity}. The following result compares $\Delta$ with the asymptotic variance of $\theta_n$ obtained in Corollary~\ref{cor:CLT}. In that result, $\|\cdot\|_{\textrm{op}}$ stands for the operator norm. That is, for any symmetric matrix $M\in\R^{d\times d}$, $\|M\|_{\textrm{op}}$ is the largest eigenvalue of $M$ in absolute values. 

\begin{proposition}
\label{thm:bound_var_asymp}
    The matrix $\Sigma-\Delta$ is positive semi-definite and 
$$\|\Sigma-\Delta\|_{\textrm{op}}\leq \frac{(\delta\lambda_d-1)^2}{2\delta\lambda_d-1}\|\Delta\|_{\textrm{op}}.$$ 
\end{proposition}


Now, we give a brief description of the asymptotic stochastic process $Y$. First, note that $Y$ is a centered, Gaussian process, thanks to the integral representation given in Proposition~\ref{prop: sol EDS}. The next result gives an estimate for the supremum of its norm on any bounded interval.

\begin{theorem}
    \label{thm:bound_sup_y} 
    There exists a universal constant $C>0$ such that for all $T>0$, 
    $$\E\big[\sup_{0<t\leq T}\|Y_t\|\big] \le C \delta\sqrt{\|\Gamma^{1/2}\|_FT}$$
    where $\|.\|_F$ denotes the Frobenius norm.
\end{theorem}

\begin{remark}
    If $B$ is a $d$-dimensional Brownian motion, then by Lemma~\ref{lem:sup_Brownien}, $\E\left[\sup_{0<t\leq T}\|\delta\Gamma^{1/2}B_t\|\right]$ is of order $\delta\sqrt{\|\Gamma^{1/2}\|_FT}$. Consequently, the bound in Theorem \ref{thm:bound_sup_y} is of the same order as the expected supremum of a Brownian motion rescaled by the diffusion coefficient of the process $Y$. Intuitively, this can be explained by the fact that the drift term in the stochastic differential equation defining $Y$ plays the role of a confining force. 
\end{remark}

\section{Conclusion}

In this work, we have established the asymptotic distribution of the whole, long-term trajectory of the SGD algorithm, under fairly minimal assumptions. Here, by long-term, we mean that we have used a topology on $\mathcal C^0((0,\infty),\R^d)$ that is agnostic to the behavior at arbitrarily small times. This is a necessary feature, since the trajectory of the SGD algorithm needs to take some time in order to forget its initialization. 

We have worked with a version of the SGD algorithm that uses decreasing step-sizes $t_n=\delta/n$, $n\geq 1$. This comes at two costs. First, these step-sizes require the knowledge of the smallest eigenvalue of $\nabla^2\Phi(\theta^*)$, since $\delta$ must be chosen larger than its inverse. This can be restrictive or even unrealistic, even though many optimization algorithms (e.g., classical descent algorithms) do require the knowledge of structural parameters of the objective function. Second, the asymptotic variance of $\theta_n$ is strictly worse than the asymptotic variance of the empirical risk minimizer. It is known that choosing larger step-sizes, that is, $t_n$ of order $n^{\alpha}$ for some $\alpha\in (1/2,1)$, and applying Polyak-Ruppert averaging along the SGD trajectory, yields asymptotic efficiency, \emph{i.e.}, it produces an estimator $\tilde\theta_n$, after $n$ steps, whose asymptotic variance coincides with that of the empirical risk minimizer \cite{article:Polyak-1992,fort2015central}. A future work is to establish an asymptotic theory for the whole long-term trajectory of such an algorithm.


\bibliographystyle{plain}
\bibliography{Bibliography}

\appendix


\section{Intermediate lemmas}

\begin{lemma} \label{lemma:convex-RS}
    Let $f:\R^d\to\R$ be convex and have a unique minimizer $x^*\in\R^d$. Let $(x_n)_{n\geq 1}$ be a sequence satisfying 
    \begin{itemize}
        \item $\|x_n-x^*\|\xrightarrow[n\to\infty]{} z$ for some $z\geq 0$ and
        \item $\sum_{n\geq 1}t_n(f(x_n)-f(x^*)) <\infty$
    \end{itemize}
    where $(t_n)_{n\geq 1}$ is a sequence of positive numbers with $\sum_{n\geq 1}t_n=\infty$. Then, $z=0$.
\end{lemma}

\begin{proof}
    Assume that $z>0$, for the sake of contradiction. Since $f$ is convex, it is continuous. And since $x^*$ is its unique minimizer, it must hold that $\eta:=\inf_{x\in\R^d:\|x-x^*\|\geq z/2}f(x)-f(x^*)>0$. Then, for all sufficiently large $n$, $\|x_n-x^*\|\geq z/2$, hence, $f(x_n)-f(x^*)\geq \eta$, contradicting the second assumption of the lemma.
\end{proof}

\begin{lemma} \label{lemma:moments}
Let $V$ be a non-negative, integrable random variable. Then, $sP(V>s)\xrightarrow[s\to\infty]{} 0$. 
\end{lemma}

\begin{proof}
    Since $V$ is non-negative and integrable, Fubini-Tonelli's theorem implies that the map $s\geq 0\mapsto P(V>s)$ is integrable, and that $\E[V]=\int_0^\infty P(V>s)\diff s$. Moreover, this map is non-increasing, yielding, for all $s\geq 0$, that 
    $$(s/2)P(V>s) \leq \int_{s/2}^{s}P(V>u)\diff u \xrightarrow[s\to\infty]{}0,$$
    yielding the result.
\end{proof}

\begin{lemma}
    \label{lemma:unif_conv} 
    Let $K\subseteq \R^p$, for some $p\geq 1$ and $f,f_1,f_2,\ldots$ be functions defined on $K$ with values in $\R^q$, for some $q\geq 1$. The following statements are equivalent:
    \begin{enumerate}
        \item $\sup_{x\in K}\|f_n(x)-f(x)\|\xrightarrow[n\to\infty]{}0$;
        \item For all sequences $(x_n)_{n\geq 1}$ in $K$, $f_n(x_n)-f(x_n)\xrightarrow[n\to\infty]{}0$.
    \end{enumerate}
\end{lemma}

\begin{proof}
    It is obvious that the first statement implies the second one, so let us only prove the converse. 
    Assume the second statement is true and suppose, for the sake of contradiction, that $\sup_{x\in K}\|f_n(x)-f(x)\|$ does not go to $0$ as $n\to\infty$. Then, there must exist $\varepsilon>0$ and an increasing map $\phi:\N^*\to\N^*$ such that $\sup_{x\in K}\|f_{\phi(n)}(x)-f(x)\|\geq \varepsilon$, for all $n\geq 1$. In particular, for each $n\geq 1$, there must exist $y_n\in K$ satisfying $\|f_{\phi(n)}(y_n)-f(y_n)\| > \varepsilon/2$. Now, consider any sequence $(x_n)_{n\geq 1}$ in $K$ such that $x_{\phi(n)}=y_n$ for all $n\geq 1$. The sequence $\|f_n(x_n)-f(x_n)\|$ does not go to zero as $n\to\infty$ as it remains larger than $\varepsilon/2$ along a subsequence. This yields the contradiction we sought for. 
\end{proof}


\begin{lemma}
\label{lemma:caract_precompact}
    Let $K$ be a subset of $\mathcal{C}^0((0,\infty),\mathbb{R}^d)$. Then $K$ is precompact with respect to the topology induced by uniform convergence on compact intervals if and only if for every $\epsilon,T$ with $0<\epsilon<T$, $K^{\epsilon,T} := \{x_{|[\epsilon,T]}: x \in K\}$ is precompact for the infinite norm on $\mathcal{C}^0([\epsilon,T],\mathbb{R}^d)$.
\end{lemma}

\begin{proof}
    Let $K$ be a subset of $\mathcal{C}^0((0,\infty),\mathbb{R}^d)$ and suppose that for every $\epsilon,T$ such that $0<\epsilon<T$, $K^{\epsilon,T}$ is precompact for the infinite norm on $\mathcal{C}^0([\epsilon,T],\mathbb{R}^d)$. Let $(x^n)_{n\ge 1}$ be a sequence in $K$. Then, there exists a subsequence $(x^{\sigma_1(n)})_{n\ge 1}$ that converges uniformly on $[1/2,2]$. We then recursively find subsequences $(x^{\sigma_1\circ\sigma_2\circ\ldots\circ\sigma_k(n)})_{n\ge 1}$ that converges uniformly on $[2^{-k},2^k]$ for every $k\ge 1$. Define $\eta(n) = \sigma_1\circ\sigma_2\circ\ldots\circ\sigma_n(n)$ for $n\ge 1$. By construction, the diagonal subsequence $(x^{\eta(n)})_{n\ge 1}$ converges uniformly on every compact intervals of $(0,\infty)$. Hence, $K$ is precompact in $\mathcal{C}^0((0,\infty),\R^d)$.
\end{proof}

\begin{lemma} \label{lemma:weak-conv-eps}
    For every interval $I \subseteq \mathbb{R}_+$ we endow $\mathcal{C}^0(I,\mathbb{R}^d)$ with the topology induced by uniform convergence on compact intervals of $I$. Let $(Z^n)_{n\geq 1}$ be a sequence of stochastic processes in $\mathcal C^0((0,\infty),\R^d)$ and denote by $(Z^n_{|I})_{n\ge 1}$ the sequence of their restrictions to $I$. Assume that for every $\epsilon>0$, the sequence $(Z^n_{|[\epsilon,\infty)})_{n\ge 1}$ is tight in $\mathcal C^0([\varepsilon,\infty),\R^d)$. Then, $(Z^n)_{n\geq 1}$ is tight in $\mathcal C^0((0,\infty),\R^d)$.
\end{lemma}

\begin{proof}
    Consider a sequence of random variables $(Z^n)_{n\ge1}$ taking values in $\mathcal{C}^0((0,\infty),\mathbb{R}^d)$ such that for every $\epsilon>0$, the sequence $(Z^n_{|[\epsilon,\infty)})_{n\ge1}$ is tight. Fix $\nu>0$ and for every $j\ge1$, let $K_j$ be a compact subset of $\mathcal{C}^0([1/j,\infty),\mathbb{R}^d)$ such that 
    $$
    \sup_n P\big(Z^n_{|[1/j,\infty)} \in K_j^\complement\big) < \frac{\nu}{2^j}
    $$
    and define $\tilde K_j := \{z\in\mathcal C^0((0,\infty),\R^d): z_{|[1/j,\infty)} \in K_j\}$. We now set $K = \cap_{j\ge 1} \tilde K_j$. Then for any $n \ge 1$,
    \begin{equation}
    \label{eq:bound_tight_with_precompact}
        P\big(Z^n \in K^\complement\big) = P\big(Z^n \in \cup_{j\ge1}\tilde K_j^\complement\big) \le \sum_{j\ge 1}P\big(Z^n_{|[1/j,\infty)} \in K_j^\complement\big) < \nu.
    \end{equation}
    It remains to check that $K$ is precompact in $\mathcal{C}^0((0,\infty), \mathbb{R}^d)$. We will use Lemma \ref{lemma:caract_precompact}. Let $\epsilon,T$ such that $0<\epsilon<T$ and consider $K^{\epsilon,T} := \{x_{|[\epsilon,T]} \ \text{such that} \ x \in K\}$. Take a sequence $(x^n)_{n\ge1}$ in $K^{\epsilon,T}$ and $j$ an integer such that $1/j<\epsilon$. Then, for every $n\ge 1$, $x^n_{|[1/j,\infty)} \in K_j$. Hence, it has a subsequence that converges uniformly on every compact intervals, and in particular on $[\epsilon,T]$. So $K^{\epsilon,T}$ is precompact in $\mathcal{C}^0([\epsilon,T],\mathbb{R}^d)$, and by Lemma \ref{lemma:caract_precompact}, K is precompact in $\mathcal{C}^0((0,\infty), \mathbb{R}^d)$.
    
    Finally, the closure $\overline{K}$ of $K$ is compact in $\mathcal{C}^0((0,\infty), \mathbb{R}^d)$, and with equation \eqref{eq:bound_tight_with_precompact}, we obtain that
    $$
    P(Z^n \in \overline{K}) \ge 1-\nu.
    $$
\end{proof}

\section{Proofs}

\subsection{Proof of Theorem~\ref{thm:tightness}}

We need to check that for all $\varepsilon>0$, there exists $C>0$ with $P(n\|\theta_n-\theta^*\|_2^2\geq C)\leq \varepsilon$ for all large enough integers $n$. 
First, fix some $r>0$ and $\alpha>\delta^{-1}$ such that $\nabla_\theta^2\Phi\geq \alpha I_d$ for all $\theta\in B(\theta^*,r)$. Such $r$ and $\alpha$ exist thanks to Assumption~\ref{assump:Phi-b} and by definition of $\delta$. 
For all integers $k,l$ with $k\leq l$, let $A_{k:l}$ be the event where $\theta_j\in B(\theta^*,r)$ for all $j=k,\ldots,l$. Fix some integers $N\geq 1$ and $n\geq N+1$. Let $k\in\{N+1,\ldots,n\}$. Using \eqref{eq:prop1_1} and the fact that $A_{N:k}\subseteq A_{N:k-1}$ and noting that the event $A_{N:k-1}$ is $\mathcal F_{k-1}$-measurable, we have
\begin{align*}
    \E[\|\theta_k-\theta^*\|_2^2\mathds 1_{A_{N:k}}] & \leq  \E[\|\theta_{k}-\theta^*\|_2^2\mathds 1_{A_{N:k-1}}] \\
    & = \E[\E[\|\theta_{k}-\theta^*\|_2^2|\mathcal F_{k-1}]\mathds 1_{A_{N:k-1}}] \\
    & \leq \E[\|\theta_{k-1}-\theta^*\|_2^2\mathds 1_{A_{N:k-1}}]-2t_{k}\E[(\Phi(\theta_{k-1})-\Phi(\theta^*))\mathds 1_{A_{N:k-1}}]+t_{k}^2\sigma^2 \\
    & \leq \E[\|\theta_{k-1}-\theta^*\|_2^2\mathds 1_{A_{N:k-1}}]-\alpha t_{k}\E[\|\theta_{k-1}-\theta^*\|_2^2\mathds 1_{A_{N:k-1}}]+t_{k}^2\sigma^2 \\
    & = (1-\alpha t_{k})\E[\|\theta_{k-1}-\theta^*\|_2^2\mathds 1_{A_{N:k-1}}]+t_{k}^2\sigma^2 \\
    & = \left(1-\frac{\alpha\delta}{k}\right)\E[\|\theta_{k-1}-\theta^*\|_2^2\mathds 1_{A_{N:k-1}}]+\frac{\delta^2\sigma^2}{k^2}.
\end{align*}
In the last inequality above, we used that for all $\theta\in B(\theta^*,r)$, $\Phi(\theta)\geq \Phi(\theta^*)+(\alpha/2)\|\theta-\theta^*\|_2^2$. Hence, by setting $V_k=\E[\|\theta_k-\theta^*\|_2^2\mathds 1_{A_{N:k}}]$, we have obtained that
\begin{equation*}
    V_k\leq \frac{k-1-\gamma}{k}V_{k-1}+\frac{\delta^2\sigma^2}{k^2}
\end{equation*}
where we set $\gamma=\alpha\delta-1>0$.
Using the inequality $1-\gamma u\leq (1-u)^{\gamma/2}$ for all $u\in [0,1/2]$, we obtain (applying the inequality to $u=1/(k-1)$), for $k\geq N+1\geq 2$, 
\begin{align*}
    V_{k} & \leq \frac{1-\gamma/(k-1)}{1+1/(k-1)}V_{k-1}+\frac{\delta^2\sigma^2}{k^2} \\
    & \leq \frac{(1-1/(k-1))^{\gamma/2}}{1+1/(k-1)}V_{k-1}+\frac{\delta^2\sigma^2}{k^2} \\
    & = \left(\frac{k-2}{k-1}\right)^{\gamma/2} \frac{k-1}{k}V_{k-1}+\frac{\delta^2\sigma^2}{k^2}
\end{align*}
and, multiplying both sides by $(k-1)^{\gamma/2}k$,
\begin{align*}
    (k-1)^{\gamma/2}kV_{k} & \leq (k-2)^{\gamma/2}(k-1)V_{k-1}+\frac{(k-1)^{\gamma/2}}{k}\delta^2\sigma^2 \\
    & \leq (k-2)^{\gamma/2}(k-1)V_{k-1}+\frac{1}{(k-1)^{1-\gamma/2}}\delta^2\sigma^2.
\end{align*}
Summing these inequalities for $k=N+1,\ldots,n$, we obtain:
\begin{equation*}
    (n-1)^{\gamma/2}nV_n\leq (N-1)^{\gamma/2}NV_N+ Kn^{\gamma/2}\delta^2\sigma^2
\end{equation*}
for some positive constant $K$ that only depends on $\gamma$. Therefore, for all $n\geq N+1$, 
\begin{equation*}
    nV_n\leq N^{\gamma/2}V_N+\left(\frac{n}{n-1}\right)^{\gamma/2}K\delta^2\sigma^2 \leq N^{\gamma/2}V_N+K'\delta^2\sigma^2
\end{equation*}
with $K'=2^{\gamma/2}K$.
Now, fix $C>0$, to be chosen later, and denote by $A_N=\bigcap_{n\geq N}A_{N:n}$. For $n\geq N+1$, 
\begin{align*}
    P(n\|\theta_n-\theta^*\|_2^2\geq C)  & \leq P(n\|\theta_n-\theta^*\|_2^2\geq C, A_{N:n}) + P(A_{N:n}^\complement) \\
    & \leq P(n\|\theta_n-\theta^*\|_2^2\mathds 1_{A_{N:n}}\geq C)+ P(A_N^\complement) \\ 
    & \leq \frac{nV_n}{C}+ P(A_N^\complement) \\
    & \leq C^{-1}\left(N^{\gamma/2}V_N+K'\delta^2\sigma^2\right) + P(A_N^\complement) 
\end{align*}
where we used Markov's inequality in the third inequality. By Theorem~\ref{thm:consistency}, $P(A_N^\complement)\xrightarrow[N\to\infty]{}0$, hence, one can fix some $N$ guaranteeing that $P(A_N^\complement)\leq\varepsilon/2$. Moreover, one can choose $C$ large enough so as to ensure that the first term in the right hand side of the last display is at most $\varepsilon/2$, which completes the proof.

\subsection{Proof of Proposition~\ref{prop: sol EDS}}
    
Fix $\varepsilon >0$. By \cite[Theorem 7.3.3]{book:Oksendal-2003}, if $Y$ is a diffusion process on $[\varepsilon,\infty)$ with generator $(G_t)_{t\geq \varepsilon}$ then $Y$ must be a solution of the SDE
\begin{equation}
\label{eq:EDS}
    \diff Y_t = -t^{-1}HY_t\diff t+\Sigma\diff B_t, \quad \forall t\geq \varepsilon.
\end{equation}
For $\varepsilon>0$, there exists $K_\epsilon$ such that for any $t \in [\epsilon,\infty)$, the function $y \mapsto -t^{-1}Hy$ is $K_\epsilon$-Lipschitz on $\mathbb{R}^d$. Hence, \cite[Theorem 7.1]{book:LeGall} ensures that for any random variable $y_{\varepsilon}$ in $ \mathbb{R}^d$, equation \eqref{eq:EDS} has a unique solution (up to indistinguishability) on $[\varepsilon, \infty)$ started at $y_\varepsilon$. We now determine this solution.
    
Let $0<\mu_1\leq \ldots \leq \mu_d$ be the eigenvalues of $H$, which is assumed to be positive definite. Let $e_1,\ldots,e_n$ be corresponding eigenvectors. Let $Y=\sum_{i=1}^n Y_ie_i$ be a solution of \eqref{eq:EDS} starting from $Y_\varepsilon=\sum_{i=1}^d Y_\varepsilon^i e_i$ at $t=\varepsilon$. 
By Itô's lemma, we have, for all $i=1,\ldots,d$ and $t\geq \varepsilon$, 
\begin{align*}
    \diff(Y_t^it^{\mu_i}) & = \mu_i Y_t^it^{\mu_i-1}\diff t+t^{\mu_i}\diff Y_t^i \\
    & = t^{\mu_i}e_i^\top \Sigma\diff B_t.
\end{align*}
Hence, we have obtained that necessarily,  
$$Y_t^it^{\mu_i}=Y_\varepsilon^i\varepsilon^{\mu_i}+ e_i^\top \int_{\varepsilon}^t s^{\mu_i}\Sigma\diff B_s, \quad \forall t\geq \varepsilon,$$
and therefore, 
\begin{equation} \label{eq:y_epsilon}
    Y_t^i=Y_\varepsilon^i\varepsilon^{\mu_i}t^{-\mu_i}+ t^{-\mu_i}e_i^\top \int_{\varepsilon}^t s^{\mu_i}\Sigma\diff B_s, \quad \forall t\geq \varepsilon.
\end{equation}
We can check that the process $Y=\sum_{i=1}^d Y^ie_i$ with $Y^i$'s defined in \eqref{eq:y_epsilon} is indeed a solution to \eqref{eq:EDS} on $[\varepsilon,\infty)$. Therefore, once the Brownian motion is fixed, this is the unique solution starting from $Y_\varepsilon$ at $t=\varepsilon$.

Now, let $Y$ be a solution to \eqref{eq:EDS} on the whole interval $(0,\infty)$, satisfying $Y_t\xrightarrow[t\downarrow 0]{} 0$ in probability. Fix some $\varepsilon>0$. Then, the restriction of $Y$ to $[\varepsilon,\infty)$ is a solution to \eqref{eq:EDS} starting from $Y_\varepsilon$ at $t=\varepsilon$, so thanks to the previous argument, its coordinates in the basis $(e_1,\ldots,e_d)$ must satisfy \eqref{eq:y_epsilon}. Letting $\varepsilon\to 0$ (recalling that $\mu_i>0$) then yields that
\begin{equation}
    Y_t=\sum_{i=1}^d t^{-\mu_i}\left(e_i^\top \int_0^t s^{\mu_i}\Sigma^{1/2}\diff B_s\right)e_i
\end{equation}
for all $t>0$.

\subsection{Proof of Lemma~\ref{lem:cv process}}

Let $p \ge 1$. Let us show that $(Y^{n,p})_{n\ge 1}$ is relatively compact, \emph{i.e}, from any subsequence of $(Y^{n,p})_{n\ge 1}$, we can extract a further subsequence weakly converging in $\mathcal{C}^0([1/p,\infty), \mathbb{R}^d)$. The desired result will then follow from Prokhorov's Theorem \cite[Theorem 5.2]{book:Billingsley1968}.
           
For simplicity (and to avoid renumbering the terms of the sequence), let us simply show that $(Y^{n,p})_{n\geq 1}$ has a weakly converging subsequence. First, thanks to Theorem~\ref{thm:tightness}, the sequence $Y_{1/p}^{n,p}$ is tight -- recall that $Y_{1/p}^{n,p}$ is a convex combination of $\frac{\lceil n/p\rceil}{\sqrt n}(\theta_{\lceil n/p\rceil}-\theta^*)$ and $\frac{\lceil n/p\rceil-1}{\sqrt n}(\theta_{\lceil n/p\rceil-1}-\theta^*)$ where $\lceil\cdot\rceil$ stands for the upper integer part. Therefore, there exists a increasing map $\phi:\N^*\to\N^*$ and a random vector $Z^p$ in $\R^d$ such that $Y_{1/p}^{\phi(n),p}\xrightarrow[n\to\infty]{} Z^p$ in distribution. By Skorohod's representation theorem \cite[Theorem 6.7]{book:Billingsley}, one can assume that the convergence holds with probability $1$. Then, by \cite[Theorem 11.2.3]{book:StroockVaradhan},\footnote{This theorem is stated for time-homogeneous Markov chains, but its proof is easily adapted to the non-homogeneous setup.} Proposition~\ref{prop:cv coeff} yields that $(Y^{\phi(n),p})_{n\ge 1}$ converges in distribution to a stochastic process in $\mathcal C^0([1/p,\infty),\R^d)$ with generator $(G_t)_{t\geq 1/p}$ given by \eqref{eq:generator} and starting from $Z^p$.

\subsection{Proof of Corollary~\ref{cor:CLT}}

Note that the map $g\in\mathcal C^0((0,\infty),\R^d)\mapsto g(1)$ is continuous with respect to the topology induced by uniform convergence on compact sets, it follows from Theorem~\ref{thm:main} that 
$$Y_1^n\xrightarrow[n\to\infty]{} Y_1$$
in distribution, that is, 
$$\sqrt n(\theta_n-\theta^*)\xrightarrow[n\to\infty]{}Y_1$$
in distribution. Now, by Proposition~\ref{prop: sol EDS}, $Y_1$ can be written as 
$$Y_1=\delta \int_0^1 e^{-\log(s)(\delta\nabla^2\Phi(\theta^*)-I_d)}\Gamma^{1/2}\diff B_s,$$
which has the $d$-variate, centered, normal distribution with covariance matrix given by 
\begin{align*}
    \Sigma & = \delta^2 \int_0^1 e^{\log(s)(\delta\nabla^2\Phi(\theta^*)-I_d)}\Gamma e^{\log(s)(\delta\nabla^2\Phi(\theta^*)-I_d)} \diff s \\
    & = \delta^2 \int_0^1 e^{-2\log(s)/\delta}e^{\delta\log(s)\nabla^2\Phi(\theta^*)}\Gamma e^{\delta\log(s)\nabla^2\Phi(\theta^*)} \diff s \\
    & = \delta \int_0^\infty e^{u/\delta}e^{-u\nabla^2\Phi(\theta^*)}\Gamma e^{-u\nabla^2\Phi(\theta^*)}\diff u
\end{align*}
where we used the change of variables $u=-\delta\log(s)$ in the last line.

\subsection{Proof of Proposition~\ref{thm:bound_var_asymp}}

Recall that $0<\lambda_1\leq \ldots\leq\lambda_d$ are the ordered eigenvalues of $\nabla^2\Phi(\theta^*)$ and that we denote by $e_1,\ldots,e_d$ a collection of associated unit eigenvectors, so $\nabla^2\Phi(\theta^*)=\sum_{i=1}^d \lambda_i e_ie_i^\top$. Then, for all $t\in\R$, $e^{-t\nabla^2\Phi(\theta^*)}=\sum_{i=1}^d e^{-\lambda_i t}e_ie_i^\top$, so one can write
\begin{align}
    \Sigma & = \delta\int_{0}^\infty e^{t/\delta}\left(\sum_{i=1}^de^{-\lambda_i t}e_ie_i^\top\right)\Gamma \left(\sum_{j=1}^de^{-\lambda_j t}e_je_j^\top\right)\diff t \nonumber \\
    & = \delta\sum_{1\leq i,j\leq d}\left(\int_0^\infty e^{t/\delta}e^{-t(\lambda_i+\lambda_j)}\diff t\right)e_ie_i^\top \Gamma e_je_j^\top \nonumber \\
    & = \sum_{1\leq i,j\leq d}\frac{\delta}{\lambda_i+\lambda_j-1/\delta}\Gamma_{i,j}e_ie_j^\top, \label{eqn:eigen_decompSigma}
\end{align}
where we denote by $\Gamma_{i,j}=e_i^\top\Gamma e_j$, for all $i,j=1,\ldots,d$. 
On the other hand, write
\begin{align}
    \Delta & = \left(\sum_{i=1}^d \lambda_i^{-1}e_ie_i^\top\right)\Gamma \left(\sum_{j=1}^d \lambda_j^{-1}e_je_j^\top\right) \nonumber \\
    & = \sum_{1\leq i,j\leq d}\frac{\Gamma_{i,j}}{\lambda_i\lambda_j}e_ie_j^\top. \label{eqn:eigen_decompDelta}
\end{align}
Therefore, using \eqref{eqn:eigen_decompSigma} and \eqref{eqn:eigen_decompDelta}, we obtain 
\begin{align}
    \Sigma-\Delta & = \sum_{1\leq i,j\leq d} \left(\frac{\delta}{\lambda_i+\lambda_j-1/\delta}-\frac{1}{\lambda_i\lambda_j}\right)\Gamma_{i,j}e_ie_j^\top \nonumber \\
    & = \sum_{1\leq i,j\leq d} \frac{(\delta\lambda_i-1)(\delta\lambda_j-1)}{\delta\lambda_i+\delta \lambda_j-1}\frac{\Gamma_{i,j}}{\lambda_i\lambda_j}e_ie_j^\top \nonumber \\
    & =: \sum_{1\leq i,j\leq d} A_{i,j}\frac{\Gamma_{i,j}}{\lambda_i\lambda_j}e_ie_j^\top \label{eqn:eigen_decomp3}
\end{align}
where the coefficients $A_{i,j}, 1\leq i,j\leq d$ are defined in an obvious way. 
Recalling \eqref{eqn:eigen_decompDelta}, and letting $A\in\R^{d\times d}$ be the symmetric matrix whose entries are given by the $A_{i,j}$'s, \eqref{eqn:eigen_decomp3} shows that $P^\top(\Sigma-\Delta)P$ is the Hadamard product of $A$ and $P^\top\Delta P$. 

Moreover, $A$ can be written as $A=DBD$ where $D\in\R^{d\times d}$ is the diagonal matrix with entries $\delta\lambda_i-1,\  i=1,\ldots,d$ and $B\in\R^{d\times d}$ is the Cauchy matrix with entries $B_{i,j}=1/(\delta\lambda_i+\delta\lambda_j-1), \ i,j=1,\ldots,d$. Since $\lambda_i>1/\delta$ for all $i=1,\ldots,d$ by definition of $\delta$, both matrices $B$ and $D$ are positive definite, and so is $A$. Therefore, the Hadamard product of $A$ and $P^\top\Delta P$ -- which is a positive, semi-definite matrix -- is positive semi-definite, yielding that $\Sigma-\Delta$ is positive semi-definite. 
Moreover, since $A$ is positive definite, \cite[Theorem 1.4.1]{bhatia2009positive} yields that 
\begin{equation*}
    \|P^\top (\Sigma-\Delta)P\|_{\textrm{op}} \leq \left(\max_{1\leq i\leq d}A_{i,i}\right) \|P^\top\Delta P\|_{\textrm{op}},
\end{equation*}
that is, $\|\Sigma-\Delta\|_{\textrm{op}}\leq \frac{(\delta\lambda_d-1)^2}{2\delta\lambda_d-1}\|\Delta\|_{\textrm{op}}$.

\subsection{Proof of Theorem \ref{thm:bound_sup_y}}
    
Fix $T>0$. For a centered gaussian process $z$ taking values in $\mathbb{R}$, we consider the pseudo-metric $d_z$ on $[0,T]$ induced by $z$ to be $d_z(s,t) = \sqrt{\E[|z_s-z_t|^2]}$. Dudley's bound \cite[Theorem 1.4.2]{book:Talagrand} tells us that the supremum of $z$ is related to the entropy number $\mathcal{N}\left([0,T], d_z, \epsilon\right)$ defined for $\epsilon>0$ to be the minimal number of open $d_z$-balls of radius $\epsilon$ required to cover $[0,T]$. This quantity is not always easy to compute, but if we have another process $z'$ with $d_z \le d_{z'}$ and $z'$ has an explicit entropy number, then $\mathcal{N}\big([0,T], d_z, \epsilon\big) \le \mathcal{N}\big([0,T], d_{z'}, \epsilon\big)$ and 
    \begin{equation}
    \label{eq:argument_bound_supy}
        \E\Big[\underset{t \in [0,T]}{\sup}|z_t|\Big] \le 48\int_0^\infty \sqrt{\log\big(\mathcal{N}([0,T], d_{z'}, \epsilon)\big)} \mathrm{d}\epsilon.
    \end{equation}
Recalling that for $t>0$, $Y_t = \sum_{i=1}^d t^{1-\delta \lambda_i} \Big(e_i^\top\displaystyle \int_0^t s^{\delta \lambda_i-1} \delta \Gamma^{\frac{1}{2}}\mathrm{d}B_s \Big)e_i$, it suffices to bound its coordinates $Y_t(i):= t^{1-\delta \lambda_i} \Big(e_i^\top\displaystyle \int_0^t s^{\delta \lambda_i-1} \delta \Gamma^{\frac{1}{2}}\mathrm{d}B_s \Big) $ for $i=1,\cdots,d$. Fix $i\in \{1,\cdots,d\}$ and let's compute the quantity:
    \begin{equation}
    \label{eq:dvlp_dy}
    \E\big[|Y_s(i)-y_t(i)|^2\big] = \E\Big[\big|s^{1-\delta\lambda_i}e_i^\top\int_0^su^{\delta\lambda_i-1}\delta \Gamma^{\frac{1}{2}}\mathrm{d}B_u - t^{1-\delta\lambda_i}e_i^\top\int_0^tu^{\delta\lambda_i-1}\delta \Gamma^{\frac{1}{2}}\mathrm{d}B_u\big|^2\Big].
    \end{equation}
    First, letting $(g_k^ i)_{k=1}^d = \Gamma^{\frac{1}{2}} e_i$ for $i=1,\cdots,d$, we have
    \begin{align*}
        \E\Big[\big|e_i^\top\int_0^su^{\delta\lambda_i-1}\delta \Gamma^{\frac{1}{2}}\mathrm{d}B_u\big|^2\Big] &= \delta^2 \int_0^s u^{2\delta\lambda_i-2}\mathrm{d}u \sum_kg_{ik}g_{ik} \\
        &= \delta^2\frac{s^{2\delta\lambda_i-1}}{2\delta\lambda_i-1}\langle e_i, \Gamma e_i\rangle.
    \end{align*}
Plugging this into \eqref{eq:dvlp_dy} yields
    \begin{align}
        \E\left[|Y_s(i)-Y_t(i)|^2\right] & = s^{2-2\delta\lambda_i}\E\Big[\big|e_i^\top\int_0^su^{\delta\lambda_i-1}\delta \Gamma^{\frac{1}{2}}\mathrm{d}B_u\big|^2\Big] \nonumber \\ 
        & \quad\quad - 2(st)^{1-\delta\lambda_i}\E\Big[\big|e_i^\top\int_0^{\min(s,t)}u^{\delta\lambda_i-1}\delta \Gamma^{\frac{1}{2}}\mathrm{d}B_u\big|^2\Big] \nonumber\\ 
        & \quad\quad + t^{2-2\delta\lambda_i}\E\Big[\big|e_i^\top\int_0^tu^{\delta\lambda_i-1}\delta \Gamma^{\frac{1}{2}}\mathrm{d}B_u\big|^2\Big] \nonumber\\
        & \quad\quad = \frac{\delta^2}{2\delta\lambda_i-1}\langle e_i, \Gamma e_i\rangle \Big(s-2\min(s,t)^{\delta\lambda_i}\max(s,t)^{1-\delta\lambda_i}+t\Big). \label{eq:dy_beforebound_ei}
    \end{align}
Define $g_i(s,t) \coloneqq s-2\min(s,t)^{\delta\lambda_i}\max(s,t)^{1-\delta\lambda_i}+t $ and fix $s \in (0,T)$. Then for $t<s$:
     \begin{align*}
         \partial_tg_i(s,t) &= 1 - 2\delta\lambda_i\left(\frac{s}{t}\right)^{1-\delta\lambda_i} \\
         \partial_t^2g_i(s,t) &= (1-\delta\lambda_i)s^{1-\delta\lambda_i}t^{\delta\lambda_i-2}2\delta\lambda_i.
     \end{align*}
Recalling that $\delta\lambda_i>1$, we obtain that the function $t\mapsto g_i(s,t)$ is concave on $[0,s]$. Moreover, $\partial_tg_i(s,s) = 1 - 2\delta\lambda_i$, hence by concavity we obtain that $g_i(s,t)\le (2\delta\lambda_i-1)(s-t)$. By performing the same computations for $t>s$, we obtain that for every $t \in (0,T)$, $g_i(s,t)\le (2\delta\lambda_i-1)|s-t|$. Hence, substituting this bound into 
\eqref{eq:dy_beforebound_ei} gives
     \begin{equation}
     \label{eq:bound_dy_dbv}
         \E\left[|Y_s(i)-Y_t(i)|^2\right] \le \delta^2|s-t|\langle e_i, \Gamma e_i\rangle.
     \end{equation}

    To conclude, we will need the following lemma:
    \begin{lemma}
        \label{lem:prop_d_B}
        Let $K>0$ and $d$ be the metric on $[0,T]$ given by $d(s,t)=\sqrt{K|s-t|}$. Then
        \[
        \int_0^\infty \sqrt{\log\big(\mathcal{N}([0,T], d, \epsilon)\big)}\mathrm{d}\epsilon \le c \sqrt{TK}
        \]
        for some $c>0$. Moreover, $d$ is the metric associated to the Brownian motion rescaled by $K$.
    \end{lemma}
    \begin{proof}[Proof of Lemma~\ref{lem:prop_d_B}]
        For $s=4\epsilon^2/K$, $d(0,s)=\sqrt{Ks} = 2\epsilon = d(t,t+s)$ for every $t\in[0,T-s]$. Hence $\left\lceil\frac{TK}{4\epsilon^2}\right\rceil$ open $d$-balls are required to cover $[0,T]$. We then compute
        \begin{align*}
            \int_0^\infty \sqrt{\log\left(\mathcal{N}([0,T], d, \epsilon)\right)}\mathrm{d}\epsilon = \int_0^\infty \sqrt{\log\left(\left\lceil\frac{TK}{4\epsilon^2}\right\rceil\right)} \mathrm{d}\epsilon \\
            = \frac{\sqrt{TK}}{2} \int_0^\infty \frac{1}{2}\sqrt{\log\left(\lceil u \rceil\right)} u^{-3/2}\mathrm{d}u.
        \end{align*}
        By letting $c = \frac{1}{4}\displaystyle\int_0^\infty \sqrt{\log\left(\lceil u \rceil\right)} u^{-3/2}\mathrm{d}u$, we obtain the Lemma.
    \end{proof}

    Dudley's bound (Theorem 1.4.2 in \cite{book:Talagrand}) together with equation \eqref{eq:bound_dy_dbv} and the result of Lemma \ref{lem:prop_d_B} with $K = \delta^2\langle e_i, \Gamma e_i\rangle$ gives
     \begin{align*}
         \E\Big[\underset{t\in[0,T]}{\sup}\Big|t^{1-\delta\lambda_i}e_i^\top\int_0^ts^{\delta\lambda_i-1}\delta \Gamma^{\frac{1}{2}}\mathrm{d}B_s\Big|\Big] 
         \le c\delta\sqrt{T \langle e_i, \Gamma e_i\rangle} 
     \end{align*}
     where we let $c = 24\displaystyle\int_0^\infty \sqrt{\log\left(\lceil u \rceil\right)} u^{-3/2}\mathrm{d}u$.
Now, since $(e_i)_{i=1}^d$ is a orthonormal basis, 
    \begin{equation}
    \label{eq:bound_Rd_to_R}
        \E\big[\underset{t\in[0,T]}{\sup}\|Y_t\|_2\big]  \\
        \le \Big(\sum_{i=1}^d\E\Big[\underset{t\in[0,T]}{\sup}|t^{1-\delta \lambda_i}e_i^\top\int_0^ts^{\delta\lambda_i-1}\delta \Gamma^{\frac{1}{2}}\mathrm{d}B_s\Big|^2\Big] \Big)^{1/2}.
    \end{equation}
    It therefore remains to bound the second moments of the supremum of the coordinates $Y(i)$ to obtain the result of Theorem~\ref{thm:bound_sup_y}. First, for $i=1,\cdots,d$
    \begin{equation*}
        \E\big[|Y_t(i)|^2\big] = \delta^2 \frac{t^{2-2\delta \lambda_i}}{2\delta\lambda_i-1}\int_0^t s^{2\delta\lambda_i-2} \mathrm{d}s \langle e_i, \Gamma e_i\rangle =  \frac{\delta^2 t}{2\delta\lambda_i-1}\langle e_i, \Gamma e_i\rangle.
    \end{equation*}
     Define $x_t \coloneqq \underset{0<s\le t}{\sup}|Y_s(i)|$ for $t\le T$, and $\sigma^2 \coloneqq \underset{0<t\le T}{\sup}\E\big[|Y_t(i)|^2\big] = \frac{\delta^2 T}{2\delta\lambda_i-1}\langle e_i, \Gamma e_i\rangle$. Since $Y(i)$ is a centered real Gaussian process it is possible to control the deviation of its supremum from its expected value using Borel-TIS bound (Theorem 2.1.1 in \cite{book:Adler2007}). For any $u>0$,
     $$
     P\Big(x_T>\E[x_T]+u\Big) \le e^{-u^2/(2\sigma^2)}.
     $$
     Hence, we can compute
     \begin{align*}
         \E[x_T^2] &= 2\int_0^\infty uP(x_T>u)\mathrm{d}u \\
         &= 2\int_0^\infty (u-\E[x_T])P(x_T-\E[x_T]>u-\E[x_T])\mathrm{d}u + 2\E[x_T]\int_0^\infty P(x_T>u) \mathrm{d}u \\
         &\le 2 \int_{0}^\infty u e^{-u^2/(2\sigma^2)}\mathrm{d}u + 2\E[x_T]^2 \\
         & \le 2\sigma^2 +2\E[x_T]^2 \le 2\delta^2  \frac{T}{2\delta\lambda_i-1}\langle e_i, \Gamma e_i\rangle + 2c^2\delta^2T\langle e_i,\Gamma e_i\rangle.
     \end{align*}
This result together with equation \eqref{eq:bound_Rd_to_R} gives
     $$\E\big[\underset{t\in[0,T]}{\sup}\|Y_t\|_2\big] \le C\delta \sqrt{\|\Gamma^{1/2}\|_FT} $$
    where $C = \big(2+2c^2\big)^{1/2}$, $\|.\|_F$ denotes the Frobenius norm and we use the fact that $2\delta\lambda_i-1>1$ for every $i=1,\cdots d$. 

    The following Lemma shows that this bound is of order of the expectation of the supremum of a Brownian motion rescaled by $\delta\Gamma^{1/2}$. This is probably a well-known result but we give a short proof below.

\begin{lemma}
        \label{lem:sup_Brownien}
    Let $B$ be a $d$-dimensional Brownian motion and $\Sigma$ a covariance matrix of size $d$. Then there exists a universal constant $c$ such that for all $T>0$, $\sqrt{\frac{2}{\pi}\|\Sigma^{1/2}\|_FT} \le \E\big[\underset{0\le t \le T}{\sup}\|\Sigma^{1/2} B_t\|\big] \le c \sqrt{\|\Sigma^{1/2}\|_FT}$.
\end{lemma}

\begin{proof}
    By rescaling $B$, we can assume that $T=1$ without loss of generality.
    Let us first prove the lower bound. If $X = (X_i)_{i=1}^d$ is a $d$-dimensional standard gaussian vector, $\E\big[\underset{0\le t \le 1}{\sup}\|\Sigma^{1/2} B_t\|\big] \ge \E\big[\|\Sigma^{1/2} X\|\big]$.
    Now, letting $0\le\mu_1\le\dots\le \mu_d$ be the eigen values of $\Sigma^{1/2}$, we have 
    \begin{equation*}
        \E[\|\Sigma^{1/2} X\|] = \E\Big[\big(\sum_{i=1}^d\mu_i^2 X_i^2\big)^{1/2}\Big] = \sqrt{\|\Sigma^{1/2}\|_F}\E\big[\big(\sum_{i=1}^dp_iX_i^2\big)^{1/2}\big]
    \end{equation*}
    where $p_i = \mu_i^2\big(\sum_{j=1}^d \mu_j^2\big)^{-1} = \mu_i^2\|\Sigma^{1/2}\|_F^{-1}$. Hence $$\E[\|\Sigma^{1/2} X\|]\ge \sqrt{\|\Sigma^{1/2}\|_F}
    \inf_{\substack{p_1,\ldots,p_d\ge 0 \\ \sum_{i=1}^d p_i = 1}} \E\big[\big(\sum_{i=1}^dp_iX_i^2\big)^{1/2}\big].
    $$
    For all $p=(p_1,\ldots,p_d)\in(\R_+)^d$, let $M(p)=  \E\big[\big(\sum_{i=1}^dp_iX_i^2\big)^{1/2}\big]$. For all $u\in \mathbb{R}^d$ with $u\neq 0$,
    $$
    u^\top \nabla^2M(p) u = -1/4\E\left[\left(\sum_{i=1}^du_iX_i^2\right)^2\left(\sum_{i=1}^dp_iX_i^2\right)^{-3/2}\right] <0
    $$
    so $M$ is a strictly concave function. Hence, on the simplex $\{(p_1,\ldots,p_d):p_i\geq 0, i=1,\ldots,d, p_1+\ldots+p_d=1\}$, its minimum is attained at an extreme point, that is, at a point $p$ whose coordinates are all zero except for one equal to $1$. For such a $p$, $M(p) = \E[|X_1|] = \sqrt{2/\pi}$, and the result follows.

    The upper bound is obtained with the result of Lemma \ref{lem:prop_d_B} together with the Borel-TIS bound (Theorem 2.1.1 in \cite{book:Adler2007}), following the same reasoning as in the last part of the proof of the Theorem \ref{thm:bound_sup_y}.
\end{proof}

\end{document}